\documentclass[11pt]{article}

\usepackage{amsmath}
\usepackage{amsthm}
\usepackage{amssymb}
\usepackage{amsfonts}
\usepackage{subfigure}
\usepackage{color}
\usepackage{makeidx}
\usepackage{hyperref}

\newcommand{\ignore}[1]{}

\newtheorem{theorem}{Theorem}

\newtheorem{lemma}[theorem]{Lemma}

\newcommand{\HS}{{\rm HS}}

\begin{document}

\title{Learning Boolean Halfspaces with Small Weights \\ from Membership Queries}
\author{Hasan Abasi\\ Department of Computer Science\\ Technion, Haifa, 32000 \and
Ali Z. Abdi \\ Convent of Nazareth School \\ Grade 11, Abas 7, Haifa \and Nader H. Bshouty
\\ Department of Computer Science\\ Technion, Haifa, 32000 }


%
\maketitle

\begin{abstract}
We consider the problem of proper learning a Boolean Halfspace with
integer weights $\{0,1,\ldots,t\}$ from membership queries only.
The best known algorithm for this problem is an adaptive algorithm
that asks $n^{O(t^5)}$ membership queries where the best lower bound
for the number of membership queries is $n^t$ \cite{AAB99}.

In this paper we close this gap and give an adaptive proper learning algorithm with
two rounds that asks $n^{O(t)}$
membership queries. We also give a non-adaptive proper learning algorithm that asks
$n^{O(t^3)}$ membership queries.
\end{abstract}

\section{Introduction}
We study the problem of learnability of boolean halfspace functions from membership queries \cite{A87,A88}.
Boolean halfspace is a function $f=[w_1x_1+\cdots+w_nx_n\ge u]$ from $\{0,1\}^n$ to $\{0,1\}$
where the {\it weights} $w_1,\ldots,w_n$ and the {\it threshold} $u$ are integers.
The function is $1$ if the arithmetic sum $w_1x_1+\cdots+w_nx_n$ is greater or equal to $u$
and zero otherwise. In the {\it membership query} model \cite{A87,A88} the learning
algorithm has access to a {\it membership oracle} ${\cal O}_f$, for some {\it target function} $f$,
that receives an assignment $a\in\{0,1\}^n$
and returns $f(a)$. A {\it proper learning algorithm} for
a class of functions $C$ is an algorithm that has access to ${\cal O}_f$ where $f\in C$ asks
membership queries and returns a function $g$ {\it in} $C$ that is equivalent to $f$.

The problem of learning classes from membership queries only were motivated
from many problems in different areas such as computational biology that arises
in whole-genome (DNA) shotgun sequencing~\cite{BAK01,ABK04,CK10}, DNA library screening~\cite{ND00},
multiplex PCR method of genome physical mapping~\cite{GK98}, linkage
discovery problems of artificial intelligence~\cite{CK10}, chemical
reaction problem~\cite{AA05,AC06,AC08} and signature coding problem for the multiple access
adder channels \cite{BG07}.

Another scenario that motivate the problem of learning Halfspaces is the following.
Given a set of $n$ similar looking objects
of unknown weights (or any other measure), but from some class of weights $W$. Suppose
we have a scale (or a measure instrument)
that can only indicate whether the weight of any set of objects exceeds
some unknown fixed threshold (or capacity). How many weighing do one needs in order to find the weights
(or all possible weights) of the objects.

In this paper we study the problem of proper learnability of boolean halfspace functions
with~$t+1$ different non-negative weights $W=\{0,1,\ldots,t\}$ from membership queries.
The best known algorithm for this problem is an adaptive algorithm
that asks $n^{O(t^5)}$ membership queries where the best lower bound
for the number of membership queries is $n^t$ \cite{AAB99}.

In this paper we close the above gap and give an adaptive proper learning algorithm with
two rounds that asks $n^{O(t)}$ membership queries.
We also give a non-adaptive proper learning algorithm that asks
$n^{O(t^3)}$ membership queries. All the algorithms in this
paper runs in time that is linear in the membership query complexity.

Extending such result to non-positive weights is impossible.
In \cite{AAB99} Abboud et. al. showed that in order to learn
boolean Halfspace functions with weights $W=\{-1,0,1\}$, we need
at least $O(2^{n-o(n)})$ membership queries. Therefore the algorithm
that asks all the $2^n$ queries in $\{0,1\}^n$ is optimal for this case.
Shevchenko and Zolotykh \cite{SZ98} studied halfspace function
over the domain $\{0,1,\ldots,k-1\}^n$ and no constraints on the coefficients.
They gave the lower bound
$\Omega(\log^{n-2} k)$ lower bound for learning this class from
membership queries. Heged\"us~\cite{H95} prove the upper bound
$O(\log^n k/\log\log n)$. For fixed $n$ Shevchenko and Zolotykh \cite{ZS97}
gave a polynomial time algorithm (in $\log k$) for this class.

This paper is organized as follows. In Section~2 we give some definitions and
preliminary results. In Section ~3 we show that any boolean halfspace
with polynomially bounded coefficients can be expressed by
an Automaton of polynomial size. A result that will be used in Section~4.
In Section~4 we give the two round learning algorithm and the non-adaptive algorithm.

\section{Definitions and Preliminary Results}
In this section we give some definitions and preliminary results that
will be used throughout the paper

\subsection{Main Lemma}
In this subsection we prove two main results that will be frequently used in this paper

For integers $t<r$ we denote $[t]:=\{1,2,\ldots,t\}$, $[t]_0=\{0,1,\ldots,t\}$ and
$[t,r]=\{t,t+1,\ldots,r\}$.

We first prove the following

\ignore{
\begin{lemma} \label{L1} Let $w_1,\ldots,w_m$ be a sequence of real numbers
and $R$ be a set of real numbers where $|R|=r$. If for every $k\in [m]$
we have $W_k:=\sum_{i=1}^kw_i\in R\cup\{0\}$
then there is set $S\subseteq [m]$ of size $|R|$
such that
\begin{enumerate}
\item $$\sum_{i\in S} w_i=W_m=\sum_{i=1}^mw_i.$$
\item For every $k\in S$ we have $$W_k(S):=\sum_{i\in S\cap [k]} w_i\in R.$$
\end{enumerate}
\end{lemma}
\begin{proof} It is enough to show that if $|[m]|=m>|R|$ then there is a set $S\varsubsetneqq [m]$
that satisfies 1 and 2. Since for $k\in[m]$, $W_k\in R\cup\{0\}$, by the Pigeonhole principle,
either there is $k_2$ such that $W_{k_2}=0$ or
there are $k_1<k_2$ such that $W_{k_1}=W_{k_2}\not=0$. Therefore $W_{k_2}-W_{k_1}=\sum_{i=k_1+1}^{k_2} w_i=0$
($k_1=0$ for the former case).
Now let $S=\{1,\cdots,k_1,k_2+1,\cdots,m\}$. Notice that
$W_k(S)=W_k$ for $k\le k_1$ and $W_k(S)=W_k-(W_{k_2}-W_{k_1})=W_k$ for $k>k_2$
and therefore $S$ satisfies 1 and 2.
\end{proof}}

\begin{lemma}\label{L2} Let $w_1,\ldots,w_m\in [-t,t]$ where at least one $w_j\not\in \{-t,0,t\}$
and $$\sum_{i=1}^m w_i=r\in [-t+1,t-1].$$
There is a permutation $\phi:[m]\to [m]$
such that for every $j\in [m]$, $W_j:=\sum_{i=1}^jw_{\phi(i)}\in [-t+1,t-1]$.
\end{lemma}
\begin{proof} Since
there is $j$ such that $w_j\in [-t+1,t-1]\backslash \{0\}$ we can take $\phi(1)=j$. Then $W_1=w_j\in[-t+1,t-1]$.
If there is $j_1,j_2$ such that $w_{j_1}=t$ and $w_{j_2}=-t$ we set $\phi(2)=j_1$, $\phi(3)=j_2$ if $W_1<0$
and $\phi(2)=j_2$, $\phi(3)=j_1$ if $W_1>0$. We repeat the latter
until there are either no more $t$ or no more $-t$ in the rest of the elements.

Assume that we have chosen $\phi(1),\ldots,\phi(k-1)$ such
that $W_j\in [-t+1,t-1]$ for $j\in[k-1]$. We now show
how to determine $\phi(k)$ so that $W_k\in [-t+1,t-1]$.
If $W_{k-1}=\sum_{i=1}^{k-1}w_{\phi(i)}> 0$
and there is $q\not\in\{\phi(1),\ldots,\phi(k-1)\}$ such that $w_q<0$ then
we take $\phi(k):=q$. Then $W_k=W_{k-1}+w_q\in [-t+1,t-1]$. If $W_{k-1}< 0$
and there is $q\not\in\{\phi(1),\ldots,\phi(k-1)\}$ such that $w_q>0$ then
we take $\phi(k):=q$. Then $W_k=W_{k-1}+w_q\in [-t+1,t-1]$. If for every
$q\not\in\{\phi(1),\ldots,\phi(k-1)\}$, $w_q>0$ (resp. $w_q<0$) then we can take an
arbitrary order of the other elements and we get $W_{k-1}<W_k<W_{k+1}<\cdots < W_m=r$
(resp. $W_{k-1}>W_k>W_{k+1}>\cdots > W_m=r$). If $W_{k-1}=0$ then there
must be $q\not\in\{\phi(1),\ldots,\phi(k-1)\}$ such that $w_q\in [-t+1,t-1]$.
This is because not both $t$ and $-t$ exist in the elements that are not assigned yet.
We then take $\phi(k):=q$.

This completes the proof.
\end{proof}

We now prove the first main lemma

\begin{lemma}\label{Main01} Let $w_1,\ldots,w_m\in [-t,t]$
and $$\sum_{i=1}^m w_i=r\in [-t+1,t-1].$$
There is a partition $S_1,S_2,\ldots,S_q$ of $[m]$
such that
\begin{enumerate}
\item For every $j\in [q-1]$, $\sum_{i\in S_j}w_{i}=0$.
\item $\sum_{i\in S_q}w_{i}=r$.
\item For every $j\in [q]$, $|S_j|\le 2t-1$.
\item If $r\not=0$ then $|S_q|\le 2t-2$.
\end{enumerate}
\end{lemma}
\begin{proof} If $w_1,\ldots, w_m\in \{-t,0,t\}$ then $r$ must be zero,
and the number of non-zero elements is even and half of them are equal to $t$ and the other half are equal to $-t$.
Then we can take $S_i=\{-t,t\}$ or $S_i=\{0\}$ for all $i$. Therefore we may assume
that at least one $w_j\not\in \{-t,0,t\}$.

By Lemma~\ref{L2} we may assume w.l.o.g (by reordering the elements) that
such that $W_j:=\sum_{i=1}^jw_{i}\in [-t+1,t-1]$ for all $j\in [m]$. Let $W_0=0$.
Consider $W_0,W_1,W_2,\ldots,W_{2t-1}$. By the pigeonhole principle
there is $0\le j_1<j_2\le 2t-1$ such that $W_{j_2}=W_{j_1}$ and then
$W_{j_2}-W_{j_1}=\sum_{i=j_1+1}^{j_2} w_{i}=0$. We then take
$S_1=\{j_1+1,\ldots,j_2\}$. Notice that $|S_1|=j_2-j_1\le 2t-1$.

Since $\sum_{i\not\in S_1}w_i=r$ we can repeat the above to find $S_2,S_3,\cdots$.
This can be repeated as long as $|[m]\backslash (S_1\cup S_2\cup \cdots\cup S_h)|\ge 2t-1$.
This proves $1-3$.

We now prove 4.
If $g:=|[m]\backslash (S_1\cup S_2\cup \cdots\cup S_{h})|< 2t-1$ then
define $S_{h+1}=[m]\backslash (S_1\cup S_2\cup \cdots\cup S_{h})$ and we get 4 for $q=h+1$.
If $g=2t-1$ then $W_0=0,W_1,W_2,\ldots,W_{2t-1}=r$ and since $r\not=0$ we must have
$0\le j_1<j_2\le 2t-1$ and $j_2-j_1<2t-1$ such that $W_{j_2}=W_{j_1}$. Then define
$S_{h+1}=\{j_1+1,\ldots,j_2\}$, $S_{h+2}=[m]\backslash (S_1\cup S_2\cup \cdots\cup S_{h+1})$ and $q=h+2$.
Then $|S_{h+2}|\le 2t-2$, $\sum_{i\in S_{h+1}} w_{i}=W_{j_2}-W_{j_1}=0$ and $\sum_{i\in S_{h+2}} w_{i}=r$.
\end{proof}

The following example shows that the bound $2t-2$ for the
size of set in Lemma~\ref{Main01} is tight.
Consider the $2t-2$ elements
$w_1=w_2=\cdots =w_{t-1}=t$ and $w_{t}=w_{t+1}=\cdots=w_{2t-2}=-(t-1)$.
The sum of any subset of elements is distinct. By adding the element
$w_{2t-1}=-(t-1)$ it is easy to show that the bound $2t-1$ in the lemma
is also tight.

\begin{lemma}\label{Main02} Let $(w_1,v_1),\ldots,(w_m,v_m)\in [-t,t]^2$
and $$\sum_{i=1}^m (w_i,v_i)=(r,s)\in [-t+1,t-1]^2.$$
There is $M\subseteq [m]$
such that
\begin{enumerate}
\item $\sum_{i\in M}(w_{i},v_i)=(r,s)$.
\item $|M|\le 8t^3-4t^2-2t+1$.
\end{enumerate}
\end{lemma}
\begin{proof} Since $w_1,\ldots,w_m\in [-t,t]$
and $\sum_{i=1}^m w_i=r\in [-t+1,t-1]$,
by Lemma~\ref{Main01}, there is a partition $S_1,\ldots,S_q$
of $[m]$ that satisfies the conditions $1-4$ given in the lemma.
Let $V_j=\sum_{i\in S_j}v_i$ for $j=1,\ldots,q$.
We have $$V_j\in [-t|S_j|,t|S_j|]\subseteq [-t(2t-1),t(2t-1)]\subset [-2t^2,2t^2]$$ for $j=1,\ldots,q$ and
\ignore{When $r\not=0$ we also have $|S_q|\le 2t-2$, $V_q\in [-t|S_q|,t|S_q|]\subseteq [-t(2t-2),t(2t-2)]$ and}
$$\sum_{i=1}^{q-1} V_i=s-V_q\in [-2t^2+1,2t^2-1].$$

If $s-V_q=0$ then for $M=S_q$ we have $|M|=|S_q|\le 2t-1\le 8t^3-4t^2-2t+1$ and
\begin{eqnarray*}
\sum_{i\in M}(w_i,v_i)&=& \sum_{i\in S_q}(w_i,v_i) =(r,V_q)=(r,s).
\end{eqnarray*}
Therefore we may assume that $s-V_q\not=0$.

Consider $V_1,V_2,\ldots,V_{q-1}$.
By 4 in Lemma~\ref{Main01} there is a set $Q\subseteq [q-1]$ of size at most $2(2t^2)-2=4t^2-2$
such that $\sum_{i\in Q} V_i= s-V_q$. Then for
$$M=S_q\cup \bigcup_{i\in Q} S_i$$ we have
$$|M|\le (2t-1)+(4t^2-2)(2t-1)=8t^3-4t^2-2t+1$$
and
\begin{eqnarray*}
\sum_{i\in M}(w_i,v_i)&=& \sum_{i\in S_q}(w_i,v_i) +\sum_{j\in Q}\sum_{i\in S_j} (w_i,v_i)\\
&=& (r,V_q) +\sum_{j\in Q} (0,V_j)\\
&=& (r,V_q)+(0,s-V_q)=(r,s).
\end{eqnarray*}
\end{proof}

\subsection{Boolean Functions}
For a boolean function $f(x_1,\ldots,x_n):\{0,1\}^n\to \{0,1\}$, $1\le i_1<i_2<\cdots<i_k\le n$
and $\sigma_1,\ldots,\sigma_k\in\{0,1\}$ we denote by $$f|_{x_{i_1}=\sigma_1, x_{i_2}=
\sigma_2,\cdots, x_{i_k}=\sigma_k}$$ the function $f$ when fixing the variables
$x_{i_j}$ to~$\sigma_j$ for all $j\in[k]$. For $a\in\{0,1\}^n$
we denote by $a|_{x_{i_1}=\sigma_1, x_{i_2}=
\sigma_2,\cdots, x_{i_k}=\sigma_k}$ the assignment $a$ where each $a_{i_j}$ is replaced
by $\sigma_j$ for all $j\in [k]$. We note here (and throughout the paper)
that $f|_{x_{i_1}=\sigma_1, x_{i_2}=
\sigma_2,\cdots, x_{i_k}=\sigma_k}$ is a function from $\{0,1\}^n\to \{0,1\}$ with
same variables $x_1,\ldots,x_n$ of $f$. Obviously
$$f|_{x_{i_1}=\sigma_1, x_{i_2}=
\sigma_2,\cdots, x_{i_k}=\sigma_k}(a)=f(a|_{x_{i_1}=\sigma_1, x_{i_2}=
\sigma_2,\cdots, x_{i_k}=\sigma_k}).$$
When $\sigma_1=\cdots=\sigma_k=\xi$ and $S=\{x_{i_1},\ldots,x_{i_k}\}$ we denote $$f|_{S\gets \xi}=f|_{x_{i_1}= \xi, x_{i_2}=
 \xi,\cdots, x_{i_k}= \xi}.$$
In the same way we define $a|_{S\gets \xi}$.
We denote by $0^n=(0,0,\ldots,0)\in\{0,1\}^n$ and  $1^n=(1,1,\ldots,1)\in\{0,1\}^n$.
For two assignments $a\in\{0,1\}^k$ and $b\in \{0,1\}^j$ we denote by $ab\in \{0,1\}^{k+j}$
the concatenation of the two assignments.

For two assignments $a,b\in \{0,1\}^n$ we write $a\le b$ if
for every $i$, $a_i\le b_i$. A boolean function $f:\{0,1\}^n\to\{0,1\}$ is {\it monotone} if
for every two assignments $a,b\in\{0,1\}^n$, if $a\le b$ then $f(a)\le f(b)$.
Recall that every monotone boolean function $f$ has a unique representation
as a reduced monotone DNF. That is, $f = M_1\vee M_2 \vee \cdots \vee M_s$ where
each {\it monomial} $M_i$ is an ANDs of input variables and for every
monomial $M_i$ there is a unique assignment $a^{(i)}\in\{0,1\}^n$ such that $f(a^{(i)})=1$
and for every $j\in [n]$ where $a^{(i)}_j=1$ we have $f(a^{(i)}|_{x_j= 0})=0$. We call
such assignment a {\it minterm} of the function~$f$. Notice that
every monotone DNF can be uniquely determined by its minterms.

We say that $x_i$ is {\it relevant} in $f$ if $f|_{x_i=0}\not\equiv f|_{x_i= 1}$.
Obviously, if $f$ is monotone then $x_i$ is relevant in $f$ if there is an assignment
$a$ such that $f(a|_{x_i= 0})=0$ and $f(a|_{x_i= 1})=1$. We say that $a$ is
a {\it semiminterm} of $f$ if for every $a_i=1$ either $f(a|_{x_i=0})=0$ or $x_i$ is not relevant in $f$.

For two assignments $a,b\in \{0,1\}^n$ we define the {\it distance} between $a$
and $b$ as $wt(a+b)$ where $wt$ is the Hamming weight and $+$ is the bitwise exclusive
or of assignments. The set $B(a;d)$ is the set of all assignments that are of distance
at most $d$ from $a\in\{0,1\}^n$.

\subsection{Symmetric and Nonsymmetric}
We say that a boolean function $f$ is {\it symmetric} in $x_i$ and $x_j$ if
for any $\xi_1,\xi_2\in\{0,1\}$ we have $f|_{x_i=\xi_1,x_j=\xi_2}\equiv f|_{x_i=\xi_2,x_j=\xi_1}$.
Obviously, this is equivalent to $f|_{x_i=0,x_j=1}\equiv f|_{x_i=1,x_j=0}$.
We say that $f$ is {\it nonsymmetric} in $x_i$ and $x_j$ if it is not symmetric in
$x_i$ and $x_j$. This is equivalent to $f|_{x_i=0,x_j=1}\not\equiv f|_{x_i=1,x_j=0}$.
We now prove
\begin{lemma}\label{nons}
Let $f$ be a monotone function. Then $f$ is nonsymmetric in $x_i$ and $x_j$ if
and only if there is a minterm $a$ of $f$ such that $a_i+a_j=1$ (one is $0$ and the other is $1$)
where $f(a|_{x_i=0,x_j=1})\not= f(a|_{x_i=1,x_j=0})$.
\end{lemma}
\begin{proof} Since $f$ is nonsymmetric in $x_i$ and $x_j$ we have
$f|_{x_i=0,x_j=1}\not\equiv f|_{x_i=1,x_j=0}$ and therefore there is an
assignment $a'$ such that $f|_{x_i=0,x_j=1}(a')\not= f|_{x_i=1,x_j=0}(a')$.
Suppose w.l.o.g. $f|_{x_i=0,x_j=1}(a')=0$ and $f|_{x_i=1,x_j=0}(a')=1$.
Take a minterm $a\le a'$ of $f|_{x_i=1,x_j=0}$. Notice that $a_i=a_j=0$.
Otherwise we can flip them to $0$ without changing the value of the function
$f|_{x_i=1,x_j=0}$ and then $a$ is not a minterm.
Then $f|_{x_i=1,x_j=0}(a)=1$ and since $a\le a'$,
$f|_{x_i=0,x_j=1}(a)=0$.

We now prove that $b=a|_{x_i=1,x_j=0}$ is a minterm of $f$.
Since $b|_{x_i=0}=a|_{x_i=0,x_j=0}< a|_{x_i=0,x_j=1}$
we have $f(b|_{x_i=0})<f(a|_{x_i=0,x_j=1})=f|_{x_i=0,x_j=1}(a)=0$
and therefore $f(b|_{x_i=0})=0$. For any $b_k=1$ where $k\not =i$, since
$a$ is a minterm for $f|_{x_i=1,x_j=0}$, we have
$f(b|_{x_k=0})=f|_{x_i=1,x_j=0}(a|_{x_k=0})=0$. Therefore $b$ is a minterm of $f$.
\end{proof}

We write $x_i\sim_f x_j$ when $f$ is symmetric in $x_i$ and $x_j$
and call $\sim_f$ the symmetric relation of $f$. The following folklore
result is proved for completeness

\begin{lemma} The relation $\sim_f$ is an equivalence relation.
\end{lemma}
\begin{proof} Obviously, $x_i\sim_f x_i$ and if $x_i\sim_f x_j$ then $x_j\sim_f x_i$.
Now if $x_i\sim_f x_j$ and $x_j\sim_f x_k$ then
$f|_{x_i=\xi_1,x_j=\xi_2,x_k=\xi_3}\equiv f|_{x_i=\xi_2,x_j=\xi_1,x_k=\xi_3}\equiv f|_{x_i=\xi_2,x_j=\xi_3,x_k=\xi_1}
$ $\equiv f|_{x_i=\xi_3,x_j=\xi_2,x_k=\xi_1}$ and therefore $x_i\sim_f x_k$.
\end{proof}

\subsection{Properties of Boolean Halfspaces}

A {\it Boolean Halfspace} function is a boolean function $f:\{0,1\}^n\to \{0,1\}$,
$f=[w_1x_1+w_2x_2+\cdots+w_nx_n\ge u]$ where $w_1,\ldots,w_n,u$ are integers, defined as
$f(x_1,\ldots,x_n)=1$ if $w_1x_1+w_2x_2+\cdots+w_nx_n\ge u$ and $0$ otherwise.
The numbers $w_i$, $i\in [n]$ are called the {\it weights} and $u$ is called the {\it threshold}.
The class $\HS$ is the class of all Boolean Halfspace functions. The class
$\HS_t$ is the class of all Boolean Halfspace functions with weights $w_i\in [t]_0$ and
the class $\HS_{[-t,t]}$ is the class of all Boolean Halfspace functions with weights $w_i\in [-t,t]$.
The representation of the above Boolean Halfspaces are not unique. For example,
$[3x_1+2x_2\ge 2]$ is equivalent to $[x_1+x_2\ge 1]$. We will assume that
\begin{eqnarray}\label{assump1}
\mbox{There is an assignment $a\in\{0,1\}^n$ such that $w_1a_1+\cdots+w_na_n=b$}
\end{eqnarray}
Otherwise we can replace $b$ by the minimum integer $w_1a_1+\cdots+w_na_n$ where $f(a)=1$
and get an equivalent function. Such $a$ is called a {\it strong assignment} of $f$.
If in addition $a$ is a minterm then it is called a {\it strong minterm}.

The following lemma follows from the above definitions
\begin{lemma} \label{pre} Let $f\in\HS_t$. We have
\begin{enumerate}
\item If $a$ is strong assignment of $f$
then $a$ is semiminterm of $f$.

\item If all the variables in $f$ are relevant then any semiminterm of $f$ is a minterm of $f$.
\end{enumerate}
\end{lemma}

We now prove

\begin{lemma}\label{Sym} Let $f=[w_1x_1+w_2x_2+\cdots+w_nx_n\ge u]\in \HS_t$. Then
\begin{enumerate}
\item If $w_1=w_2$ then $f$ is symmetric in $x_1$ and $x_2$.
\item If $f$ is symmetric in $x_1$ and $x_2$ then there are $w_1'$ and $w_2'$
such that $|w_1'-w_2'|\le 1$ and $f\equiv [w'_1x_1+w'_2x_2+w_3x_3\cdots+w_nx_n\ge u]\in \HS_t$.
\end{enumerate}
\end{lemma}
\begin{proof} If $w_1=w_2$ then for any assignment $z=(z_1,z_2,\ldots,z_n)$
we have $w_1z_1+w_2z_2+\cdots+w_nz_n=w_1 z_2+w_2z_1+\cdots+w_nz_n$. Therefore,
$f(0,1,x_3,\ldots,x_n)\equiv f(1,0,x_3,\ldots,x_n)$.

Suppose $w_1>w_2$.
It is enough to show that $f\equiv g:=[(w_1-1)x_1+(w_2+1)x_2+w_3x_3\cdots+w_nx_n\ge u]$.
Obviously, $f(x)=g(x)$ when $x_1=x_2=1$ or $x_1=x_2=0$.
If $f(0,1,x_3,\ldots,x_n)\equiv f(1,0,x_3,\ldots,x_n)$ then $w_1+w_3x_3+w_4x_4+\cdots+w_nx_n\ge u$
if and only if $w_2+w_3x_3+w_4x_4+\cdots+w_nx_n\ge u$ and therefore $w_1+w_3x_3+w_4x_4+\cdots+w_nx_n\ge u$
if and only if $(w_1-1)+w_3x_3+w_4x_4+\cdots+w_nx_n\ge u$ if and only if $(w_2+1)+w_3x_3+w_4x_4+\cdots+w_nx_n\ge u$.
\end{proof}

We now prove
\begin{lemma} \label{StA} Let $f\in \HS_t$. Let $a$ be any assignment such that
$f(a)=1$ and $f(a|_{x_i=0})=0$ for some $i\in [n]$.
There is a strong assignment of $f$ in $B(a,2t-2)$.
\end{lemma}
\begin{proof} Let $f=[w_1x_1+\cdots+w_nx_n\ge u]$. Since $f(a)=1$ and $f|_{x_i=0}(a)=0$, $a_i=1$ and
we have $w_1 a_1+w_2 a_2+\cdots+w_n a_n= u+u'$ where $t-1\ge u'\ge 0$. If
$u'=0$ then $a\in B(a,2t-2)$ is a strong assignment. So we may assume that $u'\not=0$.

By~(\ref{assump1})
there is an assignment $b$ where $w_1b_1+w_2 b_2+\cdots+w_n b_n=u$. Therefore
$w_1(b_1-a_1)+w_2(b_2-a_2)+\cdots +w_n(b_n-a_n)=-u'$. Since $w_i(b_i-a_i)\in [-t,t]$,
by Lemma~\ref{Main01} there
is $S\subseteq [n]$ of size at most $2t-2$ such that
$\sum_{i\in S} w_i(b_i-a_i)=-u'$. Therefore
$$u=-u'+(u+u')=\sum_{i\in S} w_i(b_i-a_i)+\sum_{i=1}^n w_ia_i=\sum_{i\in S} w_ib_i+\sum_{i\not\in S}w_ia_i.$$
Thus the assignment $c$ where $c_i=b_i$ for $i\in S$ and $c_i=a_i$ for $i\not\in S$
is a strong assignment of $f$ and $c\in B(a,2t-2)$.
\end{proof}

The following will be used to find the relevant variables
\begin{lemma} \label{Min} Let $f\in \HS_t$. Suppose $x_k$ is relevant in $f$.
Let $a$ be any assignment such that
$a_k=1$, $f(a)=1$ and $f(a|_{x_j=0})=0$ for some $j,k\in [n]$.
There is $c\in B(a,2t-2)$ such that $c_k=1$, $f(c)=1$ and $f(c|_{x_k=0})=0$.
\end{lemma}
\begin{proof} Let $f=[w_1x_1+\cdots+w_nx_n\ge u]$. Since $f(a)=1$ and $f(a|_{x_j=0})=0$ we have $a_j=1$
and $w_1 a_1+w_2 a_2+\cdots+w_n a_n= u+u'$ where $t-1\ge u'\ge 0$.
Let $b$ a minterm of $f$ such that $b_k=1$.
Since $b$ is a minterm
we have $w_1b_1+w_2 b_2+\cdots+w_n b_n=u+u''$ where $t-1\ge u''\ge 0$ and
since $f(b|_{x_k=0})=0$ we also have $u''-w_k<0$. If $u''=u'$ then we may
take $c=a$. Therefore we may assume that $u''\not=u'$.

Hence
$\sum_{i=1, i\not=k}^nw_i(b_i-a_i)=u''-u'\in [-t+1,t-1]\backslash \{0\}$. By Lemma~\ref{Main01} there
is $S\subseteq [n]\backslash\{k\}$ of size at most $2t-2$ such that
$\sum_{i\in S} w_i(b_i-a_i)=u''-u'$. Therefore
$$u+u''=\sum_{i\in S} w_i(b_i-a_i)+\sum_{i=1}^n w_ia_i=\sum_{i\in S} w_ib_i+\sum_{i\not\in S}w_ia_i.$$
Thus the assignment $c$ where $c_i=b_i$ for $i\in S$ and $c_i=a_i$ for $i\not\in S$
satisfies $c_k=a_k=1$ and $c\in B(a,2t-2)$. Since $\sum_{i=1,i\not=k}^n w_ic_i=u+u''-b_k<u$
we have $f(c|_{x_k=0})=0$.
\end{proof}

The following will be used to find the order of the weights
\begin{lemma} \label{antis} Let $f\in \HS_t$ be antisymmetric in $x_1$ and $x_2$.
For any minterm $a$ of $f$ of weight at least $2$ there
is $b\in B(a,2t+1)$ such that $b_1+b_2=1$ and $f|_{x_1=0,x_2=1}(b)\not= f|_{x_1=1,x_2=0}(b)$.
\end{lemma}
\begin{proof} Let $f=[w_1x_1+\cdots+w_nx_n\ge u]$. Assume w.l.o.g $w_1>w_2$.
By Lemma~\ref{nons} there is a minterm $c=(1,0,c_3,\ldots,c_n)$ such that
$f(c)=1$ and $f(0,1,c_3,\ldots,c_n)=0$. Then $W_1:=w_1+w_3c_3+\cdots+w_nc_n=u+v$ where
$0\le v\le t-1$ and $W_2:=w_2+w_3c_3+\cdots+w_nc_n=u-z$ where $1\le z\le t-1$. In fact
$-z=v-w_1+w_2$. Since
$a$ is a minterm we have $W_3:=w_1a_1+\cdots+w_na_n=u+h$ where $0\le h\le t-1$.
It is now enough to find $b\in B(a,2t-2)$ such that either
\begin{enumerate}
\item\label{C1} $b_1=1$, $b_2=0$ and $w_1 b_1+\cdots+w_n b_n=u+v$, or
\item\label{C2} $b_1=0$, $b_2=1$ and $w_1 b_1+\cdots+w_n b_n=u-z$.
\end{enumerate}
This is because if $b_1=1$, $b_2=0$ and $w_1 b_1+\cdots+w_n b_n=u+v$
(the other case is similar) then $f(1,0,b_2,\ldots,b_n)=1$ and
since $w_1\cdot 0+w_2\cdot 1+w_3\cdot a_3\cdots+w_na_n=u+v-w_1+w_2=u-z$
we have $f(0,1,b_2,\ldots,b_n)=0$.

We now
have four cases

\noindent
{\bf Case I.} $a_1=1$ and $a_2=0$: Then $W_1-W_3=w_3(c_3-a_3)+\cdots+w_n(c_n-a_n)=v-h\in [-t+1,t-1]\backslash\{0\}$.
By Lemma~\ref{Main01} there is $S\subseteq [3,n]$ of size at most $2t-1$ such that
$\sum_{i\in S} w_i(c_i-a_i)=v-h$. Therefore
$$u+v=v-h+W_3=\sum_{i\in S} w_i(c_i-a_i)+\sum_{i=1}^n w_ia_i=\sum_{i\in S} w_ic_i+\sum_{i\not\in S}^n w_ia_i.$$
Now define $b$ to be $b_i=c_i$ for $i\in S$ and $b_i=a_i$ for $i\not\in S$. Since $1,2\not\in S$
$b_1=a_1=1$ and $b_2=a_2=0$. Since
$b\in B(a,2t-1)\subset B(a,2t+1)$ and $b$ satisfies~\ref{C1}. the result follows for this case.

\noindent
{\bf Case II.} $a_1=0$ and $a_2=1$: Since $a$ is of weight at least $2$, we
may assume w.l.o.g that $a_3=1$. Since $a$ is a minterm $f(a)=1$ and $f(a|_{x_3=0})=0$
and therefore for $a'=a|_{x_3=0}$ we have $W_4:=w_1 a'_1+w_2 a'_2+\cdots+w_n a'_n=u-h'$ where $1\le h'\le t-1$.
Then $W_2-W_4=\sum_{i=3}^n w_i(c_i-a'_i)=h'-z\in [-t+1,t-1]$.
By Lemma~\ref{Main01} there is $S\subseteq [3,n]$ of size at most $2t-1$ such that
$\sum_{i\in S} w_i(c_i-a'_i)=h'-z$. Therefore
$$u-z=h'-z+W_4=\sum_{i\in S} w_i(c_i-a'_i)+\sum_{i=1}^n w_ia'_i=\sum_{i\in S} w_ic_i+\sum_{i\not\in S}^n w_ia'_i.$$
Now define $b$ to be $b_i=c_i$ for $i\in S$ and $b_i=a'_i$ for $i\not\in S$. Since $1,2\not\in S$
$b_1=a'_1=0$ and $b_2=a'_2=1$. Since
$b\in B(a',2t-1)\subset B(a,2t+1)$ and $b$ satisfies~\ref{C2}. the result follows for this case.

\noindent
{\bf Case III.} $a_1=1$ and $a_2=1$: Since $a$ is a minterm $f(a)=1$ and $f(a|_{x_1=0})=0$
and therefore for $a'=a|_{x_1=0}$ we have $W_4:=w_1 a'_1+w_2 a'_2+\cdots+w_n a'_n=u-h'$ where $1\le h'\le t-1$.
We now proceed exactly as in Case II.

\noindent
{\bf Case IV.} $a_1=0$ and $a_2=0$: Since $a$ is of weight at least $2$ we may assume
w.l.o.g that $a_3=1$. Since $a$ is a minterm $f(a)=1$ and $f(a|_{x_3=0})=0$
and therefore for $a'=a|_{x_3=0}$ we have $W_4:=a'_1w_1+a'_2w_2+\cdots+a'_nw_n=u-h'$ where $1\le h'\le t-1$.
If $f(a'|_{x_2=1})=0$ then proceed as in Case II to get $b\in B(a,2t+1)$
that satisfies~\ref{C2}. If $f(a'|_{x_1=1})=1$ then proceed as in Case I.
Now the case where $f(a'|_{x_2=1})=1$ and $f(a'|_{x_1=1})=0$ cannot happen since $w_1>w_2$.
\end{proof}

The following will be used for the non-adaptive algorithm
\begin{lemma} \label{nonadap} Let $f,g\in \HS_t$ be such that $f\not\Rightarrow g$.
For any minterm $b$ of $f$ there is $c\in B(b,8t^3+O(t^2))$ such that $f(c)+g(c)=1$.
\end{lemma}
\begin{proof} Let $f=[w_1x_1+\cdots+w_nx_n\ge u]$
and $g=[w_1'x_1+\cdots+w_n'x_n\ge u']$. Since $f\not\Rightarrow g$, there is $a'\in\{0,1\}^n$ such that
$f(a')=1$ and $g(a')=0$. Let $a\le a'$ be a minterm of $f$. Then $f(a)=1$ and since $a\le a'$ we also have
$g(a)=0$. Therefore $w_1a_1+\cdots+w_na_n=u+r$ where $0\le r\le t-1$ and $w_1'a_1+\cdots+w_n'a_n=u'-s$
for some integer $s\ge 1$. Since $b$ is a minterm of $f$ we have $w_1b_1+\cdots+w_nb_n=u+r'$
where $0\le r'\le t-1$. If $g(b)=0$ then take $c=b$. Otherwise, if for some $b_i=1$, $g(b|_{x_i=0})=1$ then
take $c=b|_{x_i=0}$. Therefore we may assume that $b$ is also a minterm of $g$. Thus
$w_1'b_1+\cdots+w_n'b_n=u+s'$ where $0\le s'\le t-1$.

Consider the sequence $Z_i$, $i=1,\ldots,n+s-1$ where $Z_i=(w_i(a_i-b_i),w_i'(a_i-b_i))$
for $i=1,\ldots,n$ and
$Z_i=(0,1)$ for $i=n+1,\ldots,n+s-1$.
Then
$$\sum_{i=1}^{n+s-1}Z_i= (r-r',-1-s')\in [-t,t]^2.$$ By Lemma~\ref{Main02} there
is a set $S\subseteq [n+s-1]$ of size $8t^3+O(t^2)$ such that
$\sum_{i\in S}Z_i= (r-r',-1-s')$. Therefore, there is a set $T\subseteq [n]$ of size at most $8t^3+O(t^2)$
such that $\sum_{i\in T}Z_i= (r-r',-\ell-1-s')$ for some $\ell>0$.  Therefore
$$\sum_{i\in T} w_i(a_i-b_i) =r-r' \mbox{\ and\ } \sum_{i\in T} w_i'(a_i-b_i) =-\ell-1-s'.$$
Define $c$ such that $c_i=a_i$ for $i\in T$ and $c_i=b_i$ for $i\not\in T$. Then
$$\sum_{i=1}^n w_ic_i =u+r\ge u \mbox{\ and\ } \sum_{i=1}^n w_i'c_i =u'-\ell-1<u'.$$ Therefore
$f(c)=1$ and $g(c)=0$. This gives the result.
\end{proof}

\section{Boolean Halfspace and Automata}
In this section we show that functions in $\HS_{[-t,t]}$ has an automaton representation
of $poly(n,t)$ size.

\begin{lemma} Let $f_1,f_2,\ldots,f_k\in \HS_{[-t,t]}$ and $g:\{0,1\}^k\to \{0,1\}$.
Then $g(f_1,\ldots,f_k)$ can be represented with an Automaton of size $(2t)^kn^{k+1}$.
\end{lemma}
\begin{proof} Let $f_i=[w_{i,1}x_1+\cdots+w_{i,n}x_n\ge u_i]$, $i=1,\ldots,k$.
Define the following automaton: The alphabet of the automaton is $\{0,1\}$.
The states are $S\subseteq [n]_0\times [-tn,tn]^k$. The automaton has
$n+1$ levels. States in level $i$ are connected only to states in level $i+1$
for all $i\in [n]_0$. We denote by $S_i$ the states in level $i$. We also have $S_i\subseteq \{i\}\times [-tn,tn]^k$
so the first entry of the state indicates the level that the state belongs to.
The state $(0,(0,0,\ldots,0))$ is the initial state and
is the only state in level $0$. That is $S_0=\{(0,(0,0,\ldots,0))\}$. We now show how to connect states
in level $i$ to states in level $i+1$.
Given a state $s=(i,(W_1,W_2,\ldots,W_k))$ in $S_i$. Then the transition function
for this state is $$\delta((i,(W_1,W_2,\ldots,W_k)),0)=(i+1,(W_1,W_2,\ldots,W_k))$$
and $$\delta((i,(W_1,W_2,\ldots,W_k)),1)=(i+1,(W_1+w_{1,i+1},W_2+w_{2,i+1},\ldots,W_k+w_{k,i+1})).$$
The accept states (where the output of the automaton is $1$) are all the states $(n,(W_1,\ldots,W_k))$ where
$g([W_1\ge u_1],[W_2\ge u_2],\ldots,[W_n\ge u_n])=1$. Here $[W_i\ge u_i]=1$ if $W_i\ge u_i$ and zero otherwise.
All other states are nonaccept states (output $0$).

We now claim that the above automaton is equivalent to $g(f_1,\ldots,f_k)$.
The proof is by induction on $n$. The claim we want to prove is that the subautomaton
that starts from state $s=(i,(W_1,W_2,\ldots,W_k))$ computes a function $g_s$ that is equivalent to
the function $g(f^i_1,\ldots,f^i_k)$ where $f^i_j=[w_{j,i+1}x_{i+1}+\cdots+w_{j,n}x_n\ge u_j-W_j]$.
This immediately follows from the fact that
$$g_s|_{x_{i+1}=0}\equiv g_{\delta(s,0)},\mbox{\ \ and\ \ } g_s|_{x_{i+1}=1}\equiv g_{\delta(s,1)}.$$
It remains to prove the result for level $n$. The claim is true for the states at level $n$ because
\begin{eqnarray*}
g(f_1^n,\ldots,f^n_k)&=&g([0\ge u_1-W_1],\ldots,[0\ge u_n-W_n])\\
&=&g([W_1\ge u_1],[W_2\ge u_2],\ldots,[W_n\ge u_n]).
\end{eqnarray*}
This completes the proof.
\end{proof}

Now the following will be used in the sequel
\begin{lemma}\label{aut2} Let $f_1,f_2\in \HS_{[-t,t]}$.
There is an algorithm that runs in time $t^2n^{3}$ and
decides whether $f_1\equiv f_2$. If $f_1\not\equiv f_2$ then the
algorithm finds an assignment
$a$ such that $f_1(a)\not= f_2(a)$.
\end{lemma}
\begin{proof} We build an automaton for $f_1+f_2$. If there is no
accept state then $f_1\equiv f_2$. If there is, then any
path from the start state to an accept state defines an assignment
$a$ such that $f_1(a)\not= f_2(a)$.
\end{proof}

\section{Two Rounds and Non-adaptive Algorithm}
In this section we give a two rounds algorithm for learning $\HS_t$ that uses~$n^{O(t)}$
membership queries.

Let $f=[w_1x_1+\ldots+w_nx_n\ge u]$. If there is a minterm
of weight one then $0\le u\le t$ and then all the minterms of $f$ are of weight
at most $t$. In this case we can find all the minterms in one round by
asking all the assignments in $B(0,t)$ (all other assignments gives $0$),
finding all the relevant variables and the antisymmetric variables and move to the second
round. Therefore we may assume that all the
minterms of $f$ are of weight at least two.

Consider the set
$$A_m=\bigcup_{i,j=0}^n B(0^i1^{n-i-j}0^j,m).$$
we now prove

\begin{lemma} Let $f\in \HS_t$. The variable $x_k$ is
relevant in $f$ if and only if there is $a\in A_{2t-2}$ such that $a_k=1$, $a|_{x_k=0}\in A_{2t-1}$
and $f(a)\not= f(a|_{x_k=0})$.
\end{lemma}
\begin{proof} If $x_k$ is relevant in $f$ then $f\not\equiv 0,1$ and
therefore $f(0^{n})=0$ and $f(1^{n})=1$. Therefore there is an element $a$ in the following sequence
$$0^n,0^{k-1}10^{n-k}, 0^{k-1}1^20^{n-k-1}, \ldots, 0^{k-1}1^{n-k+1}, 0^{k-2}1^{n-k+2},\ldots,01^{n-1},1^n$$
and $j\in [n]$ such that $f(a)=1$ and $f(a|_{x_j=0})=0$.
Notice that $a_k=1$ and therefore by Lemma~\ref{Min} there is $c\in B(a,2t-2)$ such that
$c_k=1$, $f(c)=1$ and $f(c|_{x_k=0})=0$. Since $c|_{x_k=0}\in B(a,2t-1)$, the result follows.
\end{proof}

Therefore from the assignments in $A_{2t-1}$ one can determine the relevant
variables in $f$. This implies that we may assume w.l.o.g that all the variables are relevant.
This can be done by just ignoring all the nonrelevant variables
and projecting the relevant variables to new distinct variables $y_1,\ldots,y_m$.

We now show
\begin{lemma} \label{smin}If all the variables in $f\in \HS_t$ are relevant then there is
a strong minterm $a\in A_{2t-2}$ of $f$.
\end{lemma}
\begin{proof} Follows from Lemma \ref{StA} and Lemma~\ref{pre}.
\end{proof}

\begin{lemma} Let $f\in \HS_t$ and suppose
all the variables in $f$ are relevant. Suppose $f$ is antisymmetric in $x_j$ and $x_k$.
There
is $b\in B(a,4t-1)$ such that $b_1+b_2=1$ and $f|_{x_j=0,x_k=1}(b)\not= f|_{x_j=1,x_k=0}(b)$.
\end{lemma}
\begin{proof} By Lemma~\ref{smin} there is a minterm $a\in A_{2t-2}$ of $f$.
Since $wt(a)>1$, by Lemma~\ref{antis} there is $b\in B(a,2t-1)$ such that
$b_1+b_2=1$ and $f|_{x_j=0,x_k=1}(b)\not= f|_{x_j=1,x_k=0}(b)$.
Since $b\in B(a,2t+1)\subseteq A_{4t-1}$ the result follows.
\end{proof}

Therefore from the assignments in $A_{4t-3}$ one can find a
permutation $\phi$ of the variables in
$f$ such that $f\phi=[w'_1x_1+w'_2x_2+\cdots+w'_nx_n\ge u]$ and
$w'_1\le w'_2\le \cdots \le w'_n$.

This completes the first round. We now may assume w.l.o.g that
$f=[w_1x_1+\cdots+w_nx_n\ge u]$ and $1\le w_1\le w_2\le \cdots\le w_n\le t$
and all the variables are relevant. The goal of the second round
is to find $w_i\in [1,t]$ and $u\in [0,nt]$. Since we know that $1\le w_1\le w_2\le \cdots\le w_n\le t$
we have
$${n+t-1\choose t-1} nt\le n^{t+1}$$ choices. That is at most $n^{t+1}$ possible functions
in $\HS_t$. For every two such functions $f_1,f_2$ we use Lemma~\ref{aut2}
to find out if $f_1\equiv f_2$  and if not to find an assignment $a$ such that
$f_1(a)\not= f_2(a)$.  This takes time
$${n^{t+1}\choose 2} t^2n^3\le n^{2t+7}.$$
Let $B$ the set of all such assignments. Then $|B|\le n^{2t+2}$.
In the second round we ask membership queries with all the assignments in $B$.

Now notice that if $f_1(a)\not=f_2(a)$ then either $f(a)\not=f_1(a)$ or $f(a)\not= f_2(a)$.
This shows that the assignments in $B$ eliminates all the functions that
are not equivalent to the target and all the remaining functions are equivalent
to the target.

Now using Lemma~\ref{nonadap} one can replace the set $B$ by $B(b,8t^3+O(t^2))$
for any minterm $b$ of $f$. This change the algorithm to a non-adaptive algorithm.

\end{document}